\documentclass{article} 
\usepackage{iclr2027_arxiv,times}

\usepackage{amsmath,amsfonts,bm}

\def\eqref#1{equation~\ref{#1}}

\def\1{\bm{1}}

\DeclareMathAlphabet{\mathsfit}{\encodingdefault}{\sfdefault}{m}{sl}
\SetMathAlphabet{\mathsfit}{bold}{\encodingdefault}{\sfdefault}{bx}{n}

\usepackage{hyperref}
\usepackage{url}
\usepackage[utf8]{inputenc} 
\usepackage[T1]{fontenc}    
\usepackage{booktabs}       
\usepackage{amsfonts}       
\usepackage{nicefrac}       
\usepackage{microtype}      
\usepackage{xcolor}         

\usepackage{algorithm}
\usepackage{algorithmic}
\usepackage{multirow}

\usepackage{amsmath}
\usepackage{amssymb}
\usepackage{mathtools}
\usepackage{amsthm}

\newtheorem{theorem}{Theorem}
\newtheorem{corollary}{Corollary}

\definecolor{gaincolor}{RGB}{0,115,0}
\newcommand{\gain}[1]{\textcolor{gaincolor}{\scriptsize\, (+#1)}}

\title{DELA: Learning Controllable Edits for Diffusion Language Models}

\author{%
Xinsong Feng$^{1}$,
Chenan Wang$^{2}$,
Leshu Tang$^{3}$,
Johnny Liu$^{4}$,
Dong Li$^{4}$,
\\
\bfseries
Da Li$^{4}$,
Haipeng Chen$^{2}$
\\
\normalfont
$^{1}$Georgia Institute of Technology \quad
$^{2}$William \& Mary \quad
$^{3}$UCLA \quad
$^{4}$Yotta Labs
}

\iclrfinalcopy 
\begin{document}

\maketitle

\begin{abstract}

Diffusion language models (DLMs) support bidirectional denoising and iterative refinement, but generation remains largely controlled by token-level masking schedules.
This leaves draft-level refinement implicit and limits explicit control over when to insert missing content, delete redundant spans, revisit uncertain regions, or stop editing.
We introduce \textbf{Diffusion Editing Language Adapter (DELA)}, a plug-in framework that enables controllable draft refinement with a frozen DLM.
DELA operates on clean, variable-length drafts using four explicit editing actions: \textsc{Insert}, \textsc{Delete}, \textsc{Remask}, and \textsc{Stop}.
The controller determines when and where to edit the draft, while the frozen DLM fills masked regions through denoising, separating high-level edit control from token-level generation.
A naive single-phase editor would be computationally expensive because arbitrary interior edits can disrupt prefix stability and limit reliable KV-cache reuse.
To address this, DELA uses a two-phase decoding procedure: a fast, append-only drafting phase that supports KV-cache reuse, followed by a global refinement phase that enables targeted non-monotonic edits to the full draft.
We train the controller using a staged curriculum that begins with learning local edit priors from synthetic recovery tasks and mined actions that improve draft quality.
Subsequent stages refine edit selection through rollout-based outcome preferences and alignment with downstream task feedback.
Experiments on mathematical reasoning and code generation show that DELA achieves gains of up to 5.5 percentage points over direct diffusion decoding while keeping the backbone frozen.
DELA thus provides a modular approach to explicit draft-level control in diffusion language models.

\end{abstract}
\section{Introduction}

Diffusion language models (DLMs) offer a compelling alternative to autoregressive generation by enabling bidirectional context modeling and iterative refinement~\citep{li2022diffusion,austin2021structured,campbell2022continuous,lou2023discrete,gat2024discrete}.
Building on this paradigm, masked DLMs have improved generation quality, scalability, and theoretical understanding~\citep{shi2024simplified,sahoo2024simple,ou2024your,zheng2024masked,nie2025large,ye2025dream}.
Complementary work has explored reasoning-oriented post-training~\citep{zhao2025d1,tang2026wd1} and adaptive inference-time decoding~\citep{kim2025train,wang2025remasking,ben2025accelerated}.
Despite this progress, scalable masked DLMs still largely control generation through token-level denoising schedules, leaving draft-level refinement implicit.

This limitation is important because many generation errors are structural rather than purely lexical.
A draft may omit a reasoning step, contain redundant content, include an inconsistent span, or already be correct.
Effective refinement therefore requires deciding where to intervene, which edit to apply, and when to stop, rather than only choosing which masked positions to reveal.
Prior non-autoregressive and edit-based methods show that insertion, deletion, and iterative refinement are useful abstractions~\citep{lee2018deterministic,ghazvininejad2019mask,stern2019insertion,gu2019levenshtein}.
However, masked denoising alone does not provide a controllable policy for refining clean variable-length drafts.

Recent variable-length diffusion and flow models address these limitations by modifying the generative process through edit operations, insertion expectations, expansion/deletion transitions, or deletion-insertion dynamics~\citep{havasi2025edit,kim2025any,Dreamon2025,ding2026beyond}.
These approaches are powerful, but they typically make structural editing part of the learned generative dynamics.
Adopting them therefore often requires changing the model architecture, training objective, or backbone parameters, which can be costly at scale and limits the reuse of pretrained DLMs.
We therefore ask: \textit{Can we equip an existing diffusion language model with controllable draft-level refinement without modifying or retraining its denoising backbone?}

We introduce \textbf{Diffusion Editing Language Adapter (DELA)}, a plug-in framework that equips a frozen DLM with controllable draft-level refinement.
DELA keeps the diffusion backbone fixed during both training and inference, and trains only a lightweight edit controller over clean drafts.
The controller issues four actions---\textsc{Insert}, \textsc{Delete}, \textsc{Remask}, and \textsc{Stop}---while the frozen backbone realizes masked content through conditional denoising.
This design decouples high-level edit control from token-level generation.

The design of DELA follows a broader principle in modern model adaptation: reuse a strong pretrained generator and adapt its behavior through a lightweight control or alignment layer, rather than retraining the entire model.
In DELA, the frozen DLM provides token-level denoising ability, while the edit controller learns how to use this ability to improve a clean draft.
Directly learning such a policy from downstream rewards is inefficient because the action space is heterogeneous, the value of an edit is observed only after denoising, and useful feedback may require multi-step refinement.
DELA therefore uses a staged curriculum.
It first learns executable edit primitives from synthetic recovery tasks, then strengthens the warm start through mined improving-action distillation.
Later stages use rollout-labeled candidate preferences and downstream task feedback to train outcome-aware edit selection.
At inference time, DELA selects among a small set of executable edit candidates without task labels or rollout evaluation.

DELA also addresses the computational cost of structural editing.
A naive single-phase editor would allow \textsc{Insert}, \textsc{Delete}, and \textsc{Remask} throughout generation, but arbitrary interior edits make the sequence non-monotonic, break prefix stability, prevent reliable KV-cache reuse, and require repeated processing of changing full drafts.
DELA instead uses a two-phase generation process.
A \textit{fast drafting phase} first constructs an initial draft through append-only masked-block insertion, preserving cache-friendly monotonic growth while allowing local remasking within each newly generated block.
A \textit{global refinement phase} then applies the full action space to the completed draft, enabling targeted insertion, deletion, remasking, and stopping.
Thus, DELA is not a new diffusion backbone, but a lightweight draft-conditioned editing layer for controlling an already trained denoising model.

Our main contributions are as follows.
(1) We formulate draft-level refinement for DLMs as a plug-in control problem and propose DELA, an edit-controller framework that refines clean variable-length drafts via \textsc{Insert}, \textsc{Delete}, \textsc{Remask}, and \textsc{Stop} actions while keeping the denoising backbone fixed.
(2) We introduce a two-phase decoding procedure that combines cache-friendly append-only drafting with targeted non-monotonic refinement, avoiding global editing throughout generation.
(3) We develop a staged controller-learning curriculum that starts with synthetic edit recovery, strengthens the warm start with mined improving-action distillation, and uses rollout-labeled preferences and downstream feedback for outcome-aware edit selection.
(4) We show that DELA improves mathematical reasoning and code generation by up to 5.5 percentage points while keeping the diffusion backbone fixed, demonstrating the effectiveness of explicit draft-level edit control.

\section{Related Work}

\paragraph{Diffusion language models.}
Diffusion models for text generation have been explored in both continuous and discrete spaces.
Diffusion-LM~\citep{li2022diffusion} applies Gaussian diffusion to continuous word embeddings and decodes denoised embeddings back into tokens.
Discrete diffusion models instead operate directly on token spaces.
D3PM~\citep{austin2021structured} formulates discrete corruption through Markov transition matrices, including absorbing-mask transitions.
This line of work has led to masked DLMs, where the forward process corrupts tokens into \texttt{[MASK]} and the reverse model predicts clean tokens at masked positions.
Subsequent work improves this formulation through ratio estimation~\citep{lou2023discrete}, simplified masked diffusion objectives~\citep{shi2024simplified,sahoo2024simple}, and theoretical connections between absorbing diffusion and clean-token conditional distributions~\citep{ou2024your}.
DELA adopts this masked diffusion formulation and uses the pretrained model as a frozen token-level denoiser.

\textbf{Scaling and efficient generation with DLMs.}
Masked DLMs have also been scaled to competitive language models.
LLaDA~\citep{nie2025large} trains an 8B masked diffusion model from scratch and demonstrates strong instruction-following and zero-shot performance.
Dream~\citep{ye2025dream} initializes diffusion models from pretrained autoregressive models and improves reasoning performance with context-adaptive noise scheduling.
Beyond scaling, several works study how decoding affects efficiency and generation quality.
Block Diffusion~\citep{arriola2025block} enables blockwise generation and KV-cache reuse, while Fast-dLLM~\citep{wu2025fast} introduces approximate KV caching and confidence-aware parallel decoding for DLMs.
Other work studies adaptive token ordering and early stopping~\citep{kim2025train,li2026diffusion}, while ReMDM~\citep{wang2025remasking} and Tolerator~\citep{tian2025finish} revisit generated tokens through inference-time remasking without retraining the backbone.
DELA instead learns an external controller over clean drafts that can insert, delete, remask, or stop while keeping the denoising backbone frozen.

\textbf{Variable-length and edit-based generation.}
Another approach to the length and structural limitations of masked DLMs is to introduce variable-length or edit-based dynamics.
DiffusER~\citep{reid2023diffuser} models generation as iterative editing from a prototype or noisy sequence.
Edit Flows~\citep{havasi2025edit} formulates sequence generation as a continuous-time Markov chain over edit operations, including \textsc{Insert}, \textsc{Delete}, and \textsc{Substitute}, and trains the model with discrete flow matching over auxiliary alignments.
FlexMDM~\citep{kim2025any} extends masked diffusion with a joint interpolant that learns both an unmasking posterior and an insertion expectation, enabling principled variable-length masked diffusion.
DreamOn~\citep{Dreamon2025} introduces length-changing transitions such as \textsc{Expand} and \textsc{Delete} for code infilling beyond a fixed-size canvas.
DID~\citep{ding2026beyond} replaces masking with a deletion-insertion diffusion process, where the forward process deletes tokens and the reverse process inserts tokens.
OneFlow~\citep{nguyen2025oneflow} extends insertion-based Edit Flows \citep{havasi2025edit} to multimodal non-autoregressive generation.
Together, these methods enable flexible-length generation by making length-changing edits part of the diffusion or flow process.
This couples structural control with token-level denoising, making it difficult to reuse frozen DLMs as controllable editors.
DELA instead learns an external edit controller over a frozen denoiser, separating \emph{when and where to edit} from \emph{how masked content is realized}.

\vspace{-0.1em}
\section{DELA: Learning Controllable Edits}
\label{sec:method}

We introduce \textbf{Diffusion Editing Language Adapter (DELA)}, a plug-in framework that refines variable-length clean drafts through four editing actions: \textsc{Insert}, \textsc{Delete}, \textsc{Remask}, and \textsc{Stop}.
DELA decouples \emph{structural control} from \emph{semantic generation}: a learned editor decides where and how to modify the sequence, while a frozen masked DLM realizes masked content as clean text.
This turns the frozen denoiser into a controllable non-monotonic editor without modifying the diffusion backbone.
Figure~\ref{fig:dela_overview} gives an overview of the refinement loop.

\begin{figure}[t]
    \centering
    \includegraphics[width=\linewidth]{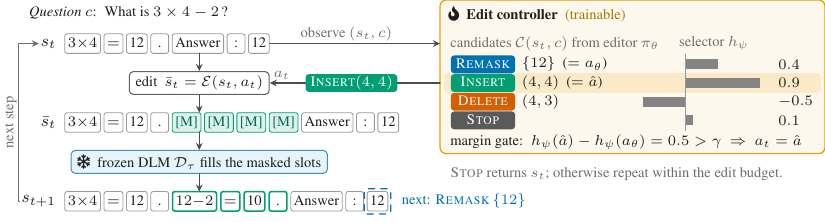}
    \vspace{-4mm}
    \caption{
    Overview of DELA. A learned edit controller operates on clean variable-length drafts and selects structural actions, while the frozen DLM realizes masked content through denoising. This design decouples high-level edit control from token-level lexical generation.
    }
    \vspace{-3mm}
    \label{fig:dela_overview}
\end{figure}

\subsection{Edit-Conditioned Refinement}
\label{subsec:edit_denoise}

DELA refines a variable-length clean draft $s_t\in V^\ast$ under context $c$.
At each time step, the editor selects an action $a_t$, applies a deterministic edit operator $\mathcal E$, and obtains an intermediate sequence
\[
    \bar s_t=\mathcal E(s_t,a_t),
\]
which may contain \texttt{[MASK]} tokens.
A frozen DLM then denoises $\bar s_t$ back to a clean state:
\begin{equation}
    s_{t+1}
    =
    T_\tau(a_t;s_t,c)
    =
    \mathcal D_\tau(\bar s_t,c),
    \label{eq:edit_denoise_transition}
\end{equation}
where $\mathcal D_\tau$ denotes threshold-based denoising with the frozen DLM backbone.
Thus, the editor controls structural changes, while lexical realization is delegated to the denoiser.
For \textsc{Stop}, the refinement process terminates and no denoising transition is applied.

\textbf{Action space.}
The editor chooses among four operations: \textsc{Insert}, \textsc{Delete}, \textsc{Remask}, and \textsc{Stop}.
For \textsc{Insert} and \textsc{Delete}, we write
\[
    a_t=(\mathrm{op}_t,\mathrm{pos}_t,\mathrm{len}_t).
\]
Here $\mathrm{op}_t\in \{\textsc{Insert}, \textsc{Delete}\}$ is the selected operation, $\mathrm{pos}_t$ denotes the start position of the edit, and
$\mathrm{len}_t$ denotes the span length.
For $\mathrm{op}_t=$ \textsc{Insert}, the edit creates $\mathrm{len}_t$ consecutive mask slots at $\mathrm{pos}_t$.
For $\mathrm{op}_t=$ \textsc{Delete}, the edit removes the span $[\mathrm{pos}_t,\mathrm{pos}_t+\mathrm{len}_t)$.
The proposal editor factorizes structural actions as
\begin{equation}
\pi_\theta(a_t\mid s_t,c)
=
\pi_\theta^{\mathrm{op}}(\mathrm{op}_t\mid s_t,c)
\pi_\theta^{\mathrm{pos}}(\mathrm{pos}_t\mid \mathrm{op}_t,s_t,c)
\pi_\theta^{\mathrm{len}}(\mathrm{len}_t\mid \mathrm{op}_t,\mathrm{pos}_t,s_t,c).
\label{eq:edit_policy_factorization}
\end{equation}

For $\mathrm{op}_t=$ \textsc{Remask}, the editor predicts token-level logits $u_i$ and samples a remask set
\begin{equation}
M_t
=
\{i:\; m_i\sim \mathrm{Bernoulli}(\sigma(u_i))\}.
\end{equation}
Tokens in $M_t$ are replaced by \texttt{[MASK]} and then regenerated by the frozen denoiser.
This operation allows DELA to revisit uncertain or inconsistent content while preserving the surrounding structure.
For $\mathrm{op}_t=$ \textsc{Stop}, no edit is applied and the current sequence is returned as the final output.

\textbf{Two-phase editing for efficient controllable refinement.}
Applying the full action space throughout generation is computationally inefficient because non-monotonic edits prevent reliable prefix reuse and may require repeatedly processing a changing full sequence.
Moreover, early generation primarily requires constructing content rather than global restructuring.
DELA therefore separates efficient draft construction from targeted non-monotonic refinement.

DELA therefore uses a two-phase procedure: a \textit{fast drafting} phase followed by a \textit{global refinement} phase.
During \emph{fast drafting}, the editor is restricted to append-only \textsc{Insert} actions with $\mathrm{pos}_t=|s_t|$.
The model constructs an initial draft by appending masked blocks and denoising them with the frozen DLM.
Before each newly appended block is finalized, DELA applies a local \textsc{Remask} step to reopen uncertain tokens and denoise them again.
Thus, fast drafting remains monotonic at the structural level while still allowing local correction within each generated block.
Once a block is finalized, its prefix KV states are kept fixed as later blocks are appended. The new block can use these cached states without recomputing the prefix.
This matches an uncached implementation that also freezes finalized prefix states, but differs from full bidirectional recomputation, which would update them using later blocks.

During \emph{global refinement}, DELA enables the full action space.
Starting from the initial draft, the editor and selector can insert missing content, delete redundant spans, remask uncertain tokens, or stop, while the frozen denoiser realizes masked content.
By restricting non-monotonic editing to this phase, DELA reserves expensive structural changes for targeted refinement.
Figure~\ref{fig:two_phase_decoding} illustrates the resulting procedure.

\begin{figure}[t]
    \centering
    \includegraphics[width=\linewidth]{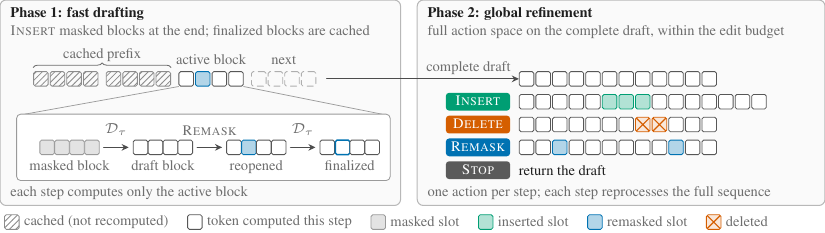}
    \vspace{-4mm}
    \caption{
    Two-phase editing in DELA. The fast drafting phase constructs an initial draft with append-only masked-block insertion and local remasking, enabling cache-friendly generation. The global refinement phase then unlocks the full action space to insert missing content, delete redundant spans, remask uncertain regions, or stop.
    }
    \vspace{-3mm}
    \label{fig:two_phase_decoding}
\end{figure}

\subsection{Learning the Edit Controller}
\label{subsec:learning_editor}

The frozen DLM provides strong token-level denoising but does not solve the control problem of deciding whether and how a clean draft should be edited.
Learning this control directly from downstream rewards is difficult because actions combine operation selection, localization, remasking, and stopping, while their effects are observed only after denoising and may unfold over multiple refinement steps.

We therefore train the controller in three stages.
Stage~1 learns \textbf{feasible} edit proposals, Stage~2 learns which proposals are \textbf{useful} under the frozen denoiser, and Stage~3 \textbf{aligns} editing with downstream task success.
The editor proposes executable actions, while a selector learns from rollout-derived preferences to predict which proposals are likely to improve the draft, replacing expensive rollout evaluation with one-step scoring at inference time.
Throughout training, the diffusion backbone remains frozen.

\textbf{Stage 1: learning executable edit proposals.}
Stage~1 initializes the proposal editor through supervised imitation using two sources of data: synthetic recovery examples and mined improving actions.
Synthetic examples teach the action structure, while mined actions bias the proposal distribution toward edits that are useful under the frozen denoiser.

First, synthetic recovery tasks provide action labels for the DELA action space.
Given a clean sequence $x$ and context $c$, we construct a set $\mathcal D_{\mathrm{syn}}$ of labeled examples $(s,c,a^\star)$.
For structural corruptions, the corrupted state $s=\xi(x)$ has an inverse edit that is either an \textsc{Insert} or a \textsc{Delete} action.
Deleting a span from $x$ induces an inverse \textsc{Insert} action, while inserting a distractor span induces an inverse \textsc{Delete} action.
Each structural example provides a target action
\[
    a^\star=(\mathrm{op}^\star,\mathrm{pos}^\star,\mathrm{len}^\star),
    \qquad
    \mathrm{op}^\star\in\{\textsc{Insert},\textsc{Delete}\}.
\]
The proposal editor imitates the full structural action with the factorized loss
\begin{equation}
\begin{aligned}
    \mathcal L_{\mathrm{struct}}(s,c,a^\star)
    =
    &-\log \pi_\theta^{\mathrm{op}}
        (\mathrm{op}^\star\mid s,c)
     -\log \pi_\theta^{\mathrm{pos}}
        (\mathrm{pos}^\star\mid \mathrm{op}^\star,s,c)  \\
    &-\log \pi_\theta^{\mathrm{len}}
        (\mathrm{len}^\star\mid \mathrm{op}^\star,\mathrm{pos}^\star,s,c).
\end{aligned}
\label{eq:struct_loss}
\end{equation}
In practice, scalar weights may be applied to the operation, position, and length terms.

For \textsc{Remask}, each synthetic example consists of a corrupted sequence $\tilde x$ and a token-level corruption mask $m^\star$.
The editor is trained to select \textsc{Remask} and localize the corrupted tokens:
\begin{equation}
\begin{aligned}
    \mathcal L_{\mathrm{remask}}(\tilde x,c,m^\star)
    =
    &-\log \pi_\theta^{\mathrm{op}}(\textsc{Remask}\mid \tilde x,c) \\
    &-
    \frac{1}{|\tilde x|}
    \sum_i
    \left[
        m_i^\star \log \sigma(u_i(\tilde x,c))
        +
        (1-m_i^\star)\log(1-\sigma(u_i(\tilde x,c)))
    \right].
\end{aligned}
\label{eq:remask_loss}
\end{equation}
For uncorrupted target sequences, the target action is \textsc{Stop}, giving
\begin{equation}
    \mathcal L_{\mathrm{stop}}(x,c)
    =
    -\log \pi_\theta^{\mathrm{op}}(\textsc{Stop}\mid x,c).
    \label{eq:stop_loss}
\end{equation}

We summarize these cases with an action-conditioned imitation loss:
\begin{equation}
\mathcal L_{\mathrm{act}}(s,c,a)
=
\begin{cases}
\mathcal L_{\mathrm{struct}}(s,c,a),
    & \mathrm{op}(a)\in\{\textsc{Insert},\textsc{Delete}\},\\
\mathcal L_{\mathrm{remask}}(s,c,M(a)),
    & \mathrm{op}(a)=\textsc{Remask},\\
\mathcal L_{\mathrm{stop}}(s,c),
    & \mathrm{op}(a)=\textsc{Stop},
\end{cases}
\label{eq:action_imitation_loss}
\end{equation}
where $M(a)$ denotes the target remask set for a \textsc{Remask} example.
Thus, each training example contributes only the loss corresponding to its labeled action.
We omit per-head scalar weights in Eq.~\eqref{eq:action_imitation_loss} for clarity.

Second, synthetic recovery teaches feasible local edits, but not all feasible edits are useful under the frozen denoiser.
We therefore mine improving actions from task-agnostic general text states.
For each sampled state, we enumerate candidate edits and execute each candidate, using the frozen denoiser whenever lexical tokens must be generated or remasked.
We keep actions that improve the resulting draft under target-aware rollout evaluation.
The target text is used only during offline label construction and is never available at inference time.
These mined actions provide outcome-informed imitation labels, but the update remains purely supervised.
The selector is not trained, and the denoiser remains frozen.
Let $\mathcal D_{\mathrm{mine}}$ denote the resulting set of mined state-action pairs $(s,c,a^+)$.

Stage~1 minimizes the mixture objective
\begin{equation}
\mathcal L_{\mathrm{stage1}}
=
\mathbb E_{(s,c,a^\star)\sim\mathcal D_{\mathrm{syn}}}
\left[
    \mathcal L_{\mathrm{act}}(s,c,a^\star)
\right]
+
\lambda_{\mathrm{mine}}
\mathbb E_{(s,c,a^+)\sim\mathcal D_{\mathrm{mine}}}
\left[
    \mathcal L_{\mathrm{act}}(s,c,a^+)
\right].
\label{eq:stage1_loss}
\end{equation}
For \textsc{Insert} and \textsc{Delete}, $\mathcal L_{\mathrm{act}}$ is the factorized action likelihood over operation, position, and length.
For \textsc{Remask}, it combines operation cross-entropy with multi-label binary cross-entropy over the target remask set.
For \textsc{Stop}, it reduces to operation cross-entropy.
We apply the objective as a curriculum, training first on synthetic recovery and then on mined improving actions.

\textbf{Stage 2: learning which edits improve a draft.}
Executable edits are not necessarily useful.
An insertion, deletion, or remasking action may be executable but still degrade the realized draft after denoising.
Stage~2 therefore performs outcome-guided policy improvement on task-agnostic general text $\mathcal D_{\mathrm{gen}}$, such as OpenWebText or WikiText.
Editing states are sampled from synthetic corruptions or short rollouts of the Stage~1 editor.
For each state, DELA builds executable candidate edits, evaluates their realized short-horizon outcomes during training, and uses the resulting preferences to improve both proposal and selection.
The clean target is used only to construct training preferences and is never available at inference time.
The rollout preferences supervise two coupled objectives: a selector-ranking objective and an editor-improvement objective.

Let $\pi_0$ denote a frozen reference editor initialized from Stage~1 or from its mined-imitation continuation, and let $\pi_\theta$ be the trainable editor initialized from $\pi_0$.
At state $(s,c)$, the current editor proposes a default executable action $a_\theta(s,c)$ by decoding its operation, position--length, or remask-set outputs.
We construct a small executable candidate set
\begin{equation}
    \mathcal C(s,c)
    =
    \{a_{\theta}(s,c), a_{\textsc{Stop}}\}
    \cup
    \mathcal C_{\mathrm{pol}}(s,c)
    \cup
    \mathcal C_{\mathrm{heur}}(s,c),
    \label{eq:candidate_set}
\end{equation}
where $\mathcal C_{\mathrm{pol}}$ contains high-probability actions decoded from $\pi_\theta$, and $\mathcal C_{\mathrm{heur}}$ contains legal boundary or span proposals.
Remask candidates are obtained from top-ranked token or block scores produced by the remask head.
Candidate construction uses only information available at inference time.
The clean target $x$ is used only to assign rollout-derived training preferences.
Heuristic candidates serve as legal proposal augmentations and diagnostics for estimating the value of available edits.

For each non-\textsc{Stop} candidate action $a$, DELA executes $s_1^a=T_\tau(a;s,c)$ and estimates its short-horizon consequence by rolling out $K$ additional steps with a fixed continuation policy $\rho$, yielding $(\tilde s_K^a,C_K^a)=\operatorname{Roll}_K^\rho(s_1^a,c)$.
For \textsc{Stop}, no edit or rollout is performed, so $s_1^a=\tilde s_K^a=s$ and $\mathrm{cost}(a)=C_K^a=0$.
In our default implementation, $\rho$ is a frozen reference editor initialized from Stage~1 or its mined-imitation continuation, and is not updated during Stage~2.
Each candidate is scored by a truncated improvement value:
\begin{equation}
Q_K^\rho(s,a)
=
F_{\mathrm{gen}}(\tilde s_K^a;x,c)
-
F_{\mathrm{gen}}(s;x,c)
-
\lambda_{\mathrm{edit}}
\left(\mathrm{cost}(a)+C_K^a\right),
\label{eq:truncated_action_value}
\end{equation}
where $F_{\mathrm{gen}}$ is a task-agnostic outcome score on general text.
In practice, we instantiate it with structural recovery and denoising-consistency metrics.

The rollout values define a soft preference target over the candidate set:
\begin{equation}
p^\star(a\mid s,c)
=
\frac{
\exp(Q_K^\rho(s,a)/\tau_{\mathrm{pref}})
}{
\sum_{b\in\mathcal C(s,c)}
\exp(Q_K^\rho(s,b)/\tau_{\mathrm{pref}})
}.
\label{eq:soft_preference_target}
\end{equation}
This target supervises both edit selection and proposal improvement.

The selector defines a candidate-normalized distribution
\begin{equation}
q_\psi(a\mid s,c,\mathcal C)
=
\frac{\exp h_\psi(s,c,a)}
{\sum_{b\in\mathcal C(s,c)}\exp h_\psi(s,c,b)},
\label{eq:selector_distribution}
\end{equation}
and is trained with
\begin{equation}
\mathcal L_{\mathrm{sel}}
=
-
\mathbb E_{(s,c)}
\sum_{a\in\mathcal C(s,c)}
p^\star(a\mid s,c)
\log q_\psi(a\mid s,c,\mathcal C).
\label{eq:selector_loss}
\end{equation}

To improve the editor itself, we normalize the editor distribution over the same candidate set:
\begin{equation}
\bar\pi_\theta(a\mid s,c,\mathcal C)
=
\frac{\pi_\theta(a\mid s,c)}
{\sum_{b\in\mathcal C(s,c)}\pi_\theta(b\mid s,c)}.
\label{eq:candidate_normalized_editor}
\end{equation}
For structural candidates, $\pi_\theta(a\mid s,c)$ is given by the operation--position--length factorization in Eq.~\ref{eq:edit_policy_factorization}.
For remask candidates, we combine the operation probability of \textsc{Remask} with the corresponding Bernoulli mask-set score as a surrogate action log-probability.
We distill the editor toward outcome-preferred actions while anchoring it to the frozen reference editor:
\begin{equation}
\mathcal L_{\mathrm{edit}}
\!=\!
-
\mathbb E_{(s,c)}
\!\!\!\!\sum_{a\in\mathcal C(s,c)}\!\!\!\!
p^\star(a\mid s,c)
\log \bar\pi_\theta(a\mid s,c,\mathcal C)
+
\lambda_{\mathrm{KL}}
D_{\mathrm{KL}}
\left(
\bar\pi_\theta(\cdot\mid s,c,\mathcal C)
\|
\bar\pi_0(\cdot\mid s,c,\mathcal C)
\right),
\label{eq:editor_improvement_loss}
\end{equation}
where $\bar\pi_0$ is the candidate-normalized distribution induced by the frozen reference editor over the same candidate set.
The Stage~2 objective is
\begin{equation}
\mathcal L_{\mathrm{stage2}}
=
\mathcal L_{\mathrm{sel}}
+
\eta_{\mathrm{edit}}\mathcal L_{\mathrm{edit}}.
\label{eq:stage2_loss}
\end{equation}
Stage~2 therefore improves the edit controller from realized short-horizon outcomes.
The selector learns conservative candidate selection, while the editor is distilled toward actions that lead to better general structural outcomes.

At inference time, DELA does not evaluate candidates by rolling them out against a target.
Instead, the selector scores one-step candidate actions before execution, and a conservative margin gate decides whether to apply the selected candidate.
Let $\hat a=\arg\max_{a\in\mathcal C(s,c)} h_\psi(s,c,a)$.
DELA executes $\hat a$ only when its score exceeds that of the default editor action $a_\theta$ by an operation-dependent margin:
\begin{equation}
a_t
=
\begin{cases}
\hat a, & h_\psi(s,c,\hat a)-h_\psi(s,c,a_\theta) > \gamma_{\mathrm{op}(a_{\theta})},\\
a_\theta, & \text{otherwise}.
\end{cases}
\label{eq:margin_gate}
\end{equation}
The gate limits deviations from the default editor action; by itself, it does not favor \textsc{Stop} over an edit.

\textbf{Stage 3: aligning edits with task success.}
General structural improvement does not always coincide with downstream success.
Stage~3 specializes DELA with task feedback while reusing the Stage~2 candidate construction, short-horizon rollout evaluation, selector-ranking objective, and editor-improvement objective.
The task-agnostic score $F_{\mathrm{gen}}$ is replaced with a task reward $F_{\mathrm{task}}$.
Depending on the task, $F_{\mathrm{task}}$ can be an exact-match or verifier score for reasoning, execution success or pass rate for code, or a learned preference reward for open-ended generation.
Task rewards are computed on candidate rollouts from $\mathcal D_{\mathrm{task}}$ during training and converted into preferences using Eq.~\eqref{eq:soft_preference_target}.
The task-level preferences align the edit controller with downstream success while keeping the diffusion backbone frozen.
At inference time, DELA uses the same conservative margin gate for task-aligned global refinement.

\textbf{Inference-time execution.}
At inference time, DELA first constructs a draft through append-only masked-block insertion with local remasking and then performs global refinement.
No target text or rollout-based candidate evaluation is used: the editor proposes one-step actions, the selector scores them, and the margin gate in Eq.~\eqref{eq:margin_gate} determines whether to override the default editor action.
Refinement continues until \textsc{Stop} is selected or the edit budget is exhausted.

\section{Experiments}
\label{sec:experiments}

We primarily evaluate DELA as a plug-in controller for frozen LLaDA-1.5 and additionally test its applicability to Dream-7B.
Because all backbone parameters remain fixed, any improvement must come from modifying edit trajectories over clean drafts rather than updating the denoising model.
Targets and task evaluators are used only for offline training and analysis, never for inference-time edit selection.

\subsection{Experimental Setup}
\label{subsec:exp_setup}

We evaluate DELA on mathematical reasoning and code generation.
For mathematical reasoning, we use GSM8K and MATH and report exact-match accuracy after answer extraction.
For code generation, we use HumanEval and MBPP and report pass@1 under the standard unit-test protocol.
For all methods, the initial draft is generated by the same frozen backbone under the same decoding budget.
Stage~3 uses task feedback to train the controller offline; the non-oracle benchmark variants decode with the trained controller fixed and without task-evaluator calls.
The proposal editor and selector contain approximately 200M trainable parameters in total, corresponding to about 2.5\% of the frozen 8B backbone.

We compare four main variants.
\textsc{Direct} returns the shared initial draft without global refinement.
This initial draft includes the same per-block local remasking used by DELA.
\textsc{Editor-only} applies the top action from the proposal editor.
\textsc{Editor+Selector} uses the learned selector to choose among proposed edit candidates without test-time rollout.
\textsc{Oracle} selects the best candidate using the task evaluator and serves only as an upper bound on the candidate set.
DELA-REMASK restricts global refinement to \textsc{Remask} and \textsc{Stop}.

\subsection{Main Results}
\label{subsec:main_results}

Table~\ref{tab:main_results} reports the downstream results with the frozen LLaDA-1.5 backbone.
Since the backbone is fixed for all methods, these comparisons isolate the effect of controlling edit trajectories.

\begin{table}[t]
\centering
\small
\setlength{\tabcolsep}{4.0pt}
\caption{
Main downstream performance with a frozen LLaDA-1.5 backbone.
All entries are percentages; numbers in parentheses denote absolute gains over \textsc{Direct}.
}
\label{tab:main_results}
\begin{tabular}{lcccc}
\toprule
Method & GSM8K & MATH & HumanEval & MBPP \\
\midrule
\textsc{Direct} & 80.5 & 36.7 & 37.2 & 39.6 \\
\textsc{Editor-only} & 83.6\gain{3.1} & 37.9\gain{1.2} & 39.0\gain{1.8} & 41.4\gain{1.8} \\
\textsc{Editor+Selector} & 85.7\gain{5.2} & 39.3\gain{2.6} & 42.7\gain{5.5} & 43.6\gain{4.0} \\
\textsc{Oracle} & 86.3\gain{5.8} & 39.8\gain{3.1} & 43.9\gain{6.7} & 44.0\gain{4.4} \\
\bottomrule
\end{tabular}
\end{table}

DELA consistently improves over direct decoding without updating the denoising model.
The gain from \textsc{Editor-only} shows that the learned editor can propose useful interventions, while the further improvement from \textsc{Editor+Selector} shows that deciding which proposed edits to execute is important.
The remaining gap to \textsc{Oracle} indicates that useful candidates are sometimes available but are not selected by the learned selector.

\subsection{Comparison with Inference-Time Refinement}
\label{subsec:inference_comparison}

We next compare DELA with alternative inference-time decoding and refinement strategies on HumanEval under the same frozen LLaDA-1.5 backbone and inference environment in Table~\ref{tab:inference_comparison}.
This comparison separates the effect of DELA's learned edit policy from differences in backbone training or model scale. NFE denotes the number of backbone forward evaluations.
The NFE comparison includes the shared drafting phase; the additional evaluations for DELA come from global refinement.

\begin{table}[t]
\centering
\small
\setlength{\tabcolsep}{4.2pt}
\caption{
Inference-time comparison on HumanEval using the same frozen LLaDA-1.5 backbone.
}
\label{tab:inference_comparison}
\begin{tabular}{lccc}
\toprule
Method & pass@1 & NFE & Latency (s) \\
\midrule
\textsc{Direct} & 37.2 & 256 & 5.43 \\
Fast-dLLM & 38.4 & 87 & 1.72 \\
Tolerator-style & 38.4 & 275 & 5.46 \\
ReMDM & 38.4 & 512 & 10.61 \\
DELA-REMASK & 40.9 & 261 & 5.57 \\
DELA & \textbf{42.7} & 264 & 5.74 \\
\bottomrule
\end{tabular}
\end{table}

The baselines expose different accuracy--cost trade-offs.
Fast-dLLM substantially reduces the denoising budget while obtaining a modest improvement over \textsc{Direct}, whereas Tolerator-style refinement and ReMDM reach the same 38.4\% pass@1 with additional denoising.
In contrast, DELA-REMASK reaches 40.9\% with only five additional NFEs over \textsc{Direct}, showing that learned targeted remasking already captures most of the refinement benefit.
Enabling the full edit space further improves performance to 42.7\% with only eight additional NFEs.
Within the tested budgets, these results show that a small number of targeted interventions can improve generation quality. They do not isolate the effect of the learned edit policy from the extra compute relative to \textsc{Direct}.

We also test whether this behavior is specific to LLaDA-1.5.
Using the same DELA formulation with a separately trained controller over frozen Dream-7B, DELA improves HumanEval from 41.5\% to 43.9\% and MBPP from 57.8\% to 60.6\%.
The smaller but consistent gains on Dream provide initial evidence that the editing formulation is applicable across different diffusion backbones rather than relying on a particular LLaDA decoding implementation.
Additional decoding comparisons and inference-time sensitivity analyses are provided in Appendix~\ref{app:additional_decoding} and Appendix~\ref{app:inference_sensitivity}.

\subsection{Controller and Editing Analysis}
\label{subsec:controller_analysis}

We analyze which components drive DELA's improvement and how the learned controller behaves in practice.
Table~\ref{tab:component_analysis} summarizes the main system-level ablations on HumanEval.

\begin{table}[t]
\centering
\small
\setlength{\tabcolsep}{6.0pt}
\caption{
Component analysis on HumanEval with the frozen LLaDA-1.5 backbone.
}
\label{tab:component_analysis}
\begin{tabular}{lcc}
\toprule
Configuration & \# Correct & pass@1 \\
\midrule
Stage~1 & 18 & 11.0 \\
Stage~1 + Stage~2 & 61 & 37.2 \\
DELA-REMASK & 67 & 40.9 \\
DELA & 70 & 42.7 \\
\textsc{Oracle} & 72 & 43.9 \\
\bottomrule
\end{tabular}
\end{table}

\textbf{Training stages.}
Stage~1 teaches the controller to construct executable edits, but not whether an intervention is beneficial on naturally generated drafts.
When applied directly to HumanEval, the Stage~1 policy edits every example, breaking 43 previously correct drafts while fixing none.
Stage~2 recovers the performance of \textsc{Direct}, suggesting that task-agnostic outcome supervision helps avoid the harmful edits made by the Stage~1 policy.
Stage~3 then aligns the outcome-aware policy with downstream task success, yielding the full DELA model.
Together with the gap between \textsc{Editor-only} and \textsc{Editor+Selector} in Table~\ref{tab:main_results}, these results show that useful proposals alone are insufficient; DELA must also learn which edits are worth executing.

\textbf{Edit behavior.}
DELA-REMASK improves HumanEval from 37.2\% to 40.9\%, while the full action space further improves performance to 42.7\%.
The REMASK-only result suggests that targeted remasking accounts for most of the gain.
This pattern is reflected in the executed actions: across HumanEval, DELA performs 24 REMASK, 2 INSERT, and 1 DELETE actions.
Overall, DELA edits only 15 of 164 outputs (9.1\%).
Among edited drafts, the median intervention changes 12 tokens, approximately 6\% of the draft.
These results indicate that DELA primarily performs sparse, targeted repair rather than broad rewriting.

\textbf{Two-phase efficiency.}
The two-phase design also limits the cost of refinement.
In a separate batch-size-one profile, two-phase drafting and refinement requires 0.461\,s per sample, compared with 0.823\,s when global edit control is invoked after every generated block.
Delaying non-monotonic refinement until after draft construction therefore avoids repeatedly invoking the global controller during generation.
Additional diagnostics for proposal quality, edit selection, edit behavior, and two-phase efficiency are provided in Appendix~\ref{app:additional}.

\section{Conclusion}
\label{sec:conclusion}

We introduced \textbf{DELA}, a plug-in framework for bringing explicit draft-level editing to frozen diffusion language models.
DELA keeps the pretrained diffusion backbone fixed and learns an external editor that refines clean drafts through structural edit actions, thereby decoupling draft-level control from token-level denoising.
To make the editing policy learnable, we design a staged curriculum that first teaches local edit recovery and then incorporates outcome-guided signals to favor edits that improve downstream task performance.
Experiments on reasoning and code generation show that DELA can improve over direct diffusion decoding without retraining the underlying backbone.
Our study is limited to two diffusion backbones and a focused task set; broader evaluation across scales and generation settings remains future work.

\clearpage
\subsection*{AI use statement}

We used generative AI tools primarily for language polishing, including improving grammar, clarity, and readability of the manuscript.
The research ideas, methodology, experiments, analysis, and conclusions were developed by the authors.
All AI-assisted text was reviewed and edited by the authors, who take full responsibility for the final content of the paper.

\bibliography{dllm}
\bibliographystyle{iclr2027_conference}

\newpage
\appendix

\section{Theoretical Insights}

\subsection{Why Singleton-Remask Consistency Works}
\label{app:qgen_theory}

We now analyze the remask-based denoising-consistency score $Q_{\mathrm{gen}}$, a component of the task-agnostic score $F_{\mathrm{gen}}$ used in the main text.
The main practical score uses general multi-token or blockwise remask patterns, but for analysis we focus on the singleton-remask case, where exactly one position is remasked at a time.
This setting admits the cleanest statistical interpretation: the corresponding population objective is exactly the conditional pseudo-log-likelihood of the data distribution.
We then show why this objective is a principled surrogate for joint likelihood: for strictly positive discrete distributions, matching all full-conditionals is sufficient to recover the same joint distribution.

\paragraph{Setup.}
Fix a conditioning context $c$ and a length-$n$ sequence $\mathbf{x}=(x_1,\dots,x_n)\in\mathcal{V}^n$.
Let $q^\star(\mathbf{x}\mid c)$ denote the data distribution, assumed strictly positive:
\[
q^\star(\mathbf{x}\mid c) > 0
\qquad
\text{for all } \mathbf{x}\in\mathcal{V}^n.
\]
For each position $i\in\{1,\dots,n\}$, let $\mathbf{x}_{\setminus i}$ denote the partially observed sequence obtained by masking out the $i$-th token.

Assume that the frozen absorbing diffusion backbone induces singleton-remask reconstruction distributions
\[
p_\phi(x_i \mid \mathbf{x}_{\setminus i}, c),
\]
which we interpret as model full-conditionals.

Define the singleton-remask score
\begin{equation}
Q_{\mathrm{gen}}^{(1)}(\mathbf{x},c)
=
\frac{1}{n}\sum_{i=1}^n
\log p_\phi(x_i \mid \mathbf{x}_{\setminus i}, c),
\label{eq:qgen_singleton}
\end{equation}
and the corresponding population objective
\begin{equation}
\mathcal{J}(\phi)
=
\mathbb{E}_{(\mathbf{x},c)\sim q^\star}
\left[
Q_{\mathrm{gen}}^{(1)}(\mathbf{x},c)
\right].
\label{eq:Jphi_def}
\end{equation}

\begin{theorem}[Population consistency of singleton-remask pseudo-likelihood]
\label{thm:singleton_pl_consistency}
For every $\phi$,
\begin{equation}
\mathcal{J}(\phi)
=
\mathcal{C}
-
\frac{1}{n}\sum_{i=1}^n
\mathbb{E}_{(\mathbf{x}_{\setminus i},c)\sim q^\star}
\left[
\mathrm{KL}\!\left(
q^\star(\cdot \mid \mathbf{x}_{\setminus i},c)
\,\middle\|\,
p_\phi(\cdot \mid \mathbf{x}_{\setminus i},c)
\right)
\right],
\label{eq:singleton_pl_kl}
\end{equation}
where
\begin{equation}
\mathcal{C}
=
\frac{1}{n}\sum_{i=1}^n
\mathbb{E}_{(\mathbf{x},c)\sim q^\star}
\left[
\log q^\star(x_i \mid \mathbf{x}_{\setminus i},c)
\right]
\label{eq:C_def}
\end{equation}
is independent of $\phi$.

Hence $\mathcal{J}(\phi)\le \mathcal{C}$ for all $\phi$, with equality if and only if
\[
p_\phi(\cdot \mid \mathbf{x}_{\setminus i},c)
=
q^\star(\cdot \mid \mathbf{x}_{\setminus i},c)
\quad
\text{for all } i
\text{ and } q^\star\text{-a.e. }(\mathbf{x}_{\setminus i},c).
\]
In particular, maximizing the singleton-remask objective is exactly equivalent to maximizing the conditional pseudo-log-likelihood of the data distribution.
\end{theorem}

\begin{proof}
By definition,
\begin{align}
\mathcal{J}(\phi)
&=
\frac{1}{n}\sum_{i=1}^n
\mathbb{E}_{(\mathbf{x},c)\sim q^\star}
\left[
\log p_\phi(x_i \mid \mathbf{x}_{\setminus i},c)
\right].
\label{eq:J_expand}
\end{align}
Fix any position $i$.
Using conditional expectation,
\begin{align}
\mathbb{E}_{(\mathbf{x},c)\sim q^\star}
\left[
\log p_\phi(x_i \mid \mathbf{x}_{\setminus i},c)
\right]
&=
\mathbb{E}_{(\mathbf{x}_{\setminus i},c)\sim q^\star}
\left[
\sum_{v\in\mathcal{V}}
q^\star(v \mid \mathbf{x}_{\setminus i},c)
\log p_\phi(v \mid \mathbf{x}_{\setminus i},c)
\right].
\label{eq:tower}
\end{align}
Add and subtract the data conditional log-probability:
\begin{align}
&\sum_{v\in\mathcal{V}}
q^\star(v \mid \mathbf{x}_{\setminus i},c)
\log p_\phi(v \mid \mathbf{x}_{\setminus i},c)
\nonumber\\
&=
\sum_{v\in\mathcal{V}}
q^\star(v \mid \mathbf{x}_{\setminus i},c)
\log q^\star(v \mid \mathbf{x}_{\setminus i},c)
\nonumber\\
&\qquad
-
\mathrm{KL}\!\left(
q^\star(\cdot \mid \mathbf{x}_{\setminus i},c)
\,\middle\|\,
p_\phi(\cdot \mid \mathbf{x}_{\setminus i},c)
\right).
\label{eq:kl_identity}
\end{align}
Substituting Eq.~\eqref{eq:kl_identity} into Eq.~\eqref{eq:tower} and then into Eq.~\eqref{eq:J_expand} yields Eq.~\eqref{eq:singleton_pl_kl}.
The equality condition follows from the nonnegativity of KL divergence.
\end{proof}

\begin{corollary}[Why pseudo-likelihood works as a surrogate for joint likelihood]
\label{cor:why_pl_works}
Assume furthermore that the model family can represent the data full-conditionals and that, for each context $c$ with $q^\star(c)>0$, the family of model full-conditionals
\[
\left\{
p_\phi(x_i \mid \mathbf{x}_{\setminus i},c)
\right\}_{i=1}^n
\]
is compatible with a unique strictly positive joint distribution
$p_\phi(\mathbf{x}\mid c)$.
If $\phi^\star$ maximizes $\mathcal{J}(\phi)$, then
\[
p_{\phi^\star}(\mathbf{x}\mid c)=q^\star(\mathbf{x}\mid c)
\qquad
\text{for all }\mathbf{x}\in\mathcal{V}^n\text{ and }c\text{ with }q^\star(c)>0.
\]
Therefore, although singleton-remask training optimizes only local full-conditionals, it identifies the same target joint distribution as joint maximum likelihood.
\end{corollary}

\begin{proof}
By the realizability assumption, some parameter attains $\mathcal J(\phi)=\mathcal C$.
Theorem~\ref{thm:singleton_pl_consistency} shows that $\mathcal C$ is an upper bound, so any maximizer $\phi^\star$ also attains it and matches the data full-conditionals:
\[
p_{\phi^\star}(x_i \mid \mathbf{x}_{\setminus i},c)
=
q^\star(x_i \mid \mathbf{x}_{\setminus i},c)
\]
for every $i$ and $q^\star$-almost every $(\mathbf{x},c)$.

It remains to show why equality of all full-conditionals determines the joint distribution.
Fix a reference state
\[
\mathbf{x}^{(0)}=(a_1,\dots,a_n)\in\mathcal{V}^n.
\]
For any target state $\mathbf{x}=(x_1,\dots,x_n)$, define the path
\[
\mathbf{x}^{(k)}=(x_1,\dots,x_k,a_{k+1},\dots,a_n),
\qquad k=0,1,\dots,n.
\]
Then $\mathbf{x}^{(0)}$ and $\mathbf{x}^{(n)}=\mathbf{x}$ are connected by a sequence of single-coordinate changes.
For any strictly positive joint distribution $p$, we have
\begin{align}
\frac{p(\mathbf{x})}{p(\mathbf{x}^{(0)})}
&=
\prod_{k=1}^n
\frac{p(\mathbf{x}^{(k)})}{p(\mathbf{x}^{(k-1)})}
\nonumber\\
&=
\prod_{k=1}^n
\frac{
p\!\left(x_k \mid x_1,\dots,x_{k-1},a_{k+1},\dots,a_n\right)
}{
p\!\left(a_k \mid x_1,\dots,x_{k-1},a_{k+1},\dots,a_n\right)
}.
\label{eq:ratio_path}
\end{align}
Thus, all probability ratios relative to the reference state are completely determined by the full-conditionals.

Applying Eq.~\eqref{eq:ratio_path} to both $p_{\phi^\star}(\mathbf{x}\mid c)$ and $q^\star(\mathbf{x}\mid c)$, and using the fact that their full-conditionals are equal, yields
\[
\frac{p_{\phi^\star}(\mathbf{x}\mid c)}{p_{\phi^\star}(\mathbf{x}^{(0)}\mid c)}
=
\frac{q^\star(\mathbf{x}\mid c)}{q^\star(\mathbf{x}^{(0)}\mid c)}
\qquad
\text{for all }\mathbf{x}\in\mathcal{V}^n.
\]
Hence there exists a constant $Z(c)>0$ such that
\[
p_{\phi^\star}(\mathbf{x}\mid c)=Z(c)\,q^\star(\mathbf{x}\mid c)
\qquad
\text{for all }\mathbf{x}\in\mathcal{V}^n.
\]
Since both distributions are normalized,
\[
1=\sum_{\mathbf{x}}p_{\phi^\star}(\mathbf{x}\mid c)
=Z(c)\sum_{\mathbf{x}}q^\star(\mathbf{x}\mid c)
=Z(c),
\]
so $Z(c)=1$.
Therefore
\[
p_{\phi^\star}(\mathbf{x}\mid c)=q^\star(\mathbf{x}\mid c)
\qquad
\text{for all }\mathbf{x}\in\mathcal{V}^n.
\]
\end{proof}

\paragraph{Remark.}
The theorem above explains why pseudo-likelihood works.
It does not need to equal the normalized joint log-likelihood pointwise.
Instead, its population optimum matches the data full-conditionals, and for strictly positive discrete distributions these full-conditionals determine the entire joint distribution.
This is why pseudo-likelihood has long been used as a tractable surrogate when the exact joint likelihood is difficult to compute.
This result concerns fixed-length singleton remasking; it does not establish that edits selected using the practical score improve downstream task outcomes.

In the main DELA reward, we use more general multi-token or blockwise remask
patterns:
\[
Q_{\mathrm{gen}}(s,c)
=
\mathbb{E}_{M\sim\mu(\cdot\mid s)}
\left[
\frac{1}{|M|}
\sum_{i\in M}
\log p_\phi(s_i \mid \mathrm{mask}(s,M), c)
\right].
\]
These practical variants are motivated by the singleton-remask theory above, but should be viewed as conditional-consistency surrogates rather than exact normalized joint log-likelihood estimators.

\section{Implementation Details}
\label{app:implementation}

\subsection{Experimental Configuration}
\label{app:experimental_configuration}

Unless otherwise stated, all experiments use LLaDA-1.5 as the frozen denoising backbone.
The backbone parameters are never updated.
We train only the DELA edit-control modules, including the proposal editor and the selector.
The \textsc{Direct} baseline uses the same initial-draft pipeline as DELA, including local per-block remasking, and returns the draft before global refinement.
In our current implementation, this initial draft is generated with $\texttt{gen\_length}=256$, $\texttt{steps}=256$, and $\texttt{block\_length}=16$.
NFE counts calls to the frozen backbone across the full decoding procedure, including local remasking during drafting and any subsequent global refinement. Thus, the 256 NFEs reported for \textsc{Direct} cover the shared drafting phase, while the eight additional NFEs reported for DELA arise during global refinement.
Table~\ref{tab:impl_config} summarizes the main implementation hyperparameters used by the training and decoding algorithms.

\begin{table}[htbp]
\centering
\caption{Main implementation hyperparameters for DELA.}
\label{tab:impl_config}
\begin{tabular}{ll}
\toprule
Hyperparameter & Value \\
\midrule
Denoising backbone & LLaDA-1.5 \\
Backbone update & Frozen \\
Denoising operator & $\mathcal D_\tau$ \\
Initial draft length & $\texttt{gen\_length}=256$ \\
Initial denoising iterations & $\texttt{steps}=256$ \\
Initial block length & $\texttt{block\_length}=16$ \\
Confidence threshold & $\tau=0.9$ \\
Edit refinement budget & $T_{\max}=2$ actions \\
Candidate rollout horizon & $K=1$ edit step \\
Stage2 selector margin & $\gamma_{\mathrm{op}}\in\{0.10,0.15\}$ \\
Stage3 selector margin & $\gamma_{\mathrm{op}}=0$ unless otherwise stated \\
Trainable modules & Proposal editor $\pi_\theta$, selector $h_\psi$ \\
\bottomrule
\end{tabular}
\end{table}

For conservative selector diagnostics, we additionally report fixed-margin variants.
These margins are specified in the corresponding result tables.
During final task evaluation, the selector scores executable candidates using learned logits only.
Rollout rewards, verifier feedback, target answers, and oracle candidate scores are used only for training or diagnostic upper bounds, never for test-time candidate selection.
For Stage3 task evaluation, we disable test-time rollout reranking by setting $\texttt{selector\_rollout\_rerank\_topk}=0$.

\subsection{Compute Resources}
\label{app:compute_resources}

We report the compute resources used in our experiments to support reproducibility.
All main experiments were conducted on a workstation with 8 NVIDIA RTX PRO 6000 GPUs.
The LLaDA-1.5 backbone was kept frozen throughout all experiments, and training only updated the DELA edit-control modules.
Unless otherwise stated, training and evaluation used the same machine.

\begin{table}[htbp]
\centering
\caption{Compute resources used for DELA experiments.}
\label{tab:compute_resources}
\begin{tabular}{lcc}
\toprule
Experiment & GPUs & Trainable components \\
\midrule
Stage 1: learning executable edit proposals & 8 RTX PRO 6000 & Proposal editor \\
Stage 2: learning which edits improve a draft & 8 RTX PRO 6000 & Proposal editor, selector \\
Stage 3: aligning edits with task success & 8 RTX PRO 6000 & Proposal editor, selector \\
Downstream evaluation & 8 RTX PRO 6000 & None \\
\bottomrule
\end{tabular}
\end{table}

\subsection{Training Algorithm}
\label{app:training_algorithm}

Algorithm~\ref{alg:dela_training} summarizes the DELA training curriculum.
The curriculum first learns executable edit proposals through imitation, then uses mined improving actions to strengthen the warm start.
Subsequent stages use rollout-derived preferences and downstream feedback to train outcome-aware edit selection.

\begin{algorithm}[htbp]
\caption{DELA training curriculum}
\label{alg:dela_training}
\begin{algorithmic}[1]
\REQUIRE General corpus $\mathcal D_{\mathrm{gen}}$, task corpus $\mathcal D_{\mathrm{task}}$, frozen denoiser $\mathcal D_\tau$, editor $\pi_\theta$, selector $h_\psi$
\STATE \textbf{Stage 1: synthetic edit imitation}
\STATE Construct synthetic recovery dataset $\mathcal D_{\mathrm{syn}}$ from $\mathcal D_{\mathrm{gen}}$.
\FOR{each labeled example $(s,c,a^\star)\sim\mathcal D_{\mathrm{syn}}$}
    \STATE Update $\pi_\theta$ with the action imitation loss $\mathcal L_{\mathrm{act}}(s,c,a^\star)$.
\ENDFOR

\STATE \textbf{Stage 1.5: mined improving-action distillation}
\STATE Mine improving actions $\mathcal D_{\mathrm{mine}}$ from task-agnostic text states using the frozen denoiser.
\FOR{each mined example $(s,c,a^+)\sim\mathcal D_{\mathrm{mine}}$}
    \STATE Update $\pi_\theta$ with $\mathcal L_{\mathrm{act}}(s,c,a^+)$.
\ENDFOR
\STATE Set frozen reference editor $\pi_0\leftarrow\pi_\theta$.
\STATE Initialize the trainable editor from $\pi_0$.

\STATE \textbf{Stage 2: general outcome-guided improvement}
\FOR{each editing state $(s,c,x)$ from $\mathcal D_{\mathrm{gen}}$}
    \STATE Build executable candidate set $\mathcal C(s,c)$ using Eq.~\ref{eq:candidate_set}.
    \FOR{each candidate $a\in\mathcal C(s,c)$}
        \IF{$a=\textsc{Stop}$}
            \STATE Set $s_1^a=\tilde s_K^a=s$ and $C_K^a=0$.
        \ELSE
            \STATE Execute $s_1^a=T_\tau(a;s,c)$ and roll out $(\tilde s_K^a,C_K^a)=\operatorname{Roll}_K^\rho(s_1^a,c)$ with reference policy $\rho$.
        \ENDIF
        \STATE Compute $Q_K^\rho(s,a)$ using Eq.~\ref{eq:truncated_action_value}.
    \ENDFOR
    \STATE Convert rollout values to $p^\star(a\mid s,c)$ using Eq.~\ref{eq:soft_preference_target}.
    \STATE Update $h_\psi$ and $\pi_\theta$ with Eq.~\ref{eq:stage2_loss}.
\ENDFOR

\STATE \textbf{Stage 3: task-specific alignment}
\FOR{each editing state $(s,c,x)$ from $\mathcal D_{\mathrm{task}}$}
    \STATE Repeat candidate construction and short-horizon rollout evaluation as in Stage~2.
    \STATE Replace $F_{\mathrm{gen}}$ with task reward $F_{\mathrm{task}}$.
    \STATE Update the same outcome-guided objectives with task-level preferences.
\ENDFOR
\end{algorithmic}
\end{algorithm}

\subsection{Decoding Algorithm}
\label{app:decoding}

Algorithm~\ref{alg:dela_decoding} details DELA decoding.
The procedure first constructs a draft through append-only insertion with local per-block remasking, and then performs global refinement with the full edit action space.
The fast-drafting steps, including local remask selection, are shared by \textsc{Direct} and DELA.
Each edit-conditioned denoising call has a finite iteration budget and resolves any masks remaining at its final iteration by argmax prediction.
During refinement, the selector can override the default action of the editor only when the margin gate is satisfied, keeping decoding conservative.

\begin{algorithm}[htbp]
\caption{DELA decoding with fast drafting and global refinement}
\label{alg:dela_decoding}
\begin{algorithmic}[1]
\REQUIRE Prompt/context $c$, frozen denoiser $\mathcal D_\tau$, editor $\pi_\theta$, selector $h_\psi$, draft budget $B_{\max}$, refinement budget $T_{\max}$
\STATE Initialize $s\leftarrow \emptyset$

\STATE \textbf{Fast drafting}
\FOR{$b=1,\ldots,B_{\max}$}
    \STATE Select append-only action $a_b=\textsc{Insert}(\mathrm{pos}=|s|,\mathrm{len}=\ell_b)$.
    \STATE Apply edit--denoise transition $s\leftarrow T_\tau(a_b;s,c)$.
    \STATE Let $\mathcal B_b$ denote the newly generated block.
    \STATE Predict local remask scores on $\mathcal B_b$ and select $M_b\subseteq \mathcal B_b$.
    \IF{$M_b\neq \emptyset$}
        \STATE Apply local remask transition $s\leftarrow T_\tau(\textsc{Remask}(M_b);s,c)$.
    \ENDIF
    \IF{draft stopping criterion is met}
        \STATE \textbf{break}
    \ENDIF
\ENDFOR

\STATE \textbf{Global refinement}
\FOR{$t=1,\ldots,T_{\max}$}
    \STATE Compute default editor action $a_\theta=\arg\max_a \pi_\theta(a\mid s,c)$.
    \STATE Build candidate set $\mathcal C(s,c)$ using Eq.~\ref{eq:candidate_set}.
    \STATE Ensure $a_\theta\in\mathcal C(s,c)$.
    \STATE Select $\hat a=\arg\max_{a\in\mathcal C(s,c)} h_\psi(s,c,a)$.
    \IF{$h_\psi(s,c,\hat a)-h_\psi(s,c,a_\theta)>\gamma_{\mathrm{op}(a_\theta)}$}
        \STATE $a_t\leftarrow \hat a$
    \ELSE
        \STATE $a_t\leftarrow a_\theta$
    \ENDIF
    \IF{$a_t=\textsc{Stop}$}
        \STATE \textbf{break}
    \ENDIF
    \STATE Apply edit--denoise transition $s\leftarrow T_\tau(a_t;s,c)$.
\ENDFOR

\STATE \textbf{return} $s$
\end{algorithmic}
\end{algorithm}

\section{Additional Results}
\label{app:additional}

\subsection{Stage 1: Synthetic Recovery}
\label{app:synthetic_recovery}

Stage~1 provides a supervised initialization for the proposal editor through synthetic recovery.
Given a clean sequence, we construct a corrupted draft using synthetic structural perturbations and train the editor to predict a recovery action that moves the corrupted draft back toward the clean target.
For example, deleting a span from the clean sequence induces an \textsc{Insert} recovery action, while inserting a distractor span induces a \textsc{Delete} recovery action.
This diagnostic evaluates whether the editor has learned the basic syntax and semantics of executable structural edits, including the operation type, edit position, and span length.

Table~\ref{tab:synthetic_recovery} reports the Stage~1 synthetic recovery results.
The editor achieves high operation and position accuracy, indicating that it can identify the intended edit type and approximate edit location.
Length prediction is more challenging, especially for insertion examples, because multiple nearby span lengths can often lead to similar repaired states.
For this reason, we report tolerance-based length and joint metrics in addition to exact joint accuracy.
These results should be interpreted as action-grammar diagnostics rather than evidence of downstream task improvement.

\begin{table}[t]
\centering
\caption{
Stage~1 synthetic recovery diagnostics.
The editor is evaluated on synthetic corrupted drafts and predicts the recovery action toward the clean target.
These metrics measure whether the editor learns executable edit syntax, not whether the resulting policy improves downstream generation.
}
\label{tab:synthetic_recovery}
\small
\setlength{\tabcolsep}{3pt}
\begin{tabular}{lcccccccc}
\toprule
Scope & Op acc. & Pos acc. & Op+Pos acc. & Len $\pm 1$ & Len $\pm 2$ & Joint $\pm 1$ & Joint $\pm 2$ & Exact joint \\
\midrule
Overall & 89.5 & 83.5 & 74.8 & 41.5 & 59.8 & 30.3 & 45.8 & 12.3 \\
Insert  & 95.5 & 70.0 & 67.5 & 27.5 & 43.5 & 19.0 & 31.0 & 4.0  \\
Delete  & 83.5 & 97.0 & 82.0 & 55.5 & 76.0 & 41.5 & 60.5 & 20.5 \\
\bottomrule
\end{tabular}
\end{table}

\subsection{Stage 1: Proposal Quality}
\label{app:stage1_proposal}

Synthetic recovery teaches the editor how to construct executable actions, but does not establish whether these actions are useful under the frozen denoiser.
We therefore evaluate the effect of the mined improving-action distillation used in the second part of Stage~1.

Table~\ref{tab:stage1_bridge} compares the editor trained only on synthetic recovery with the final Stage~1 editor after mined-positive distillation.
Mined-positive measures the fraction of examples for which the candidate set contains a rollout-mined improving action, while Top-1 improves measures whether the highest-probability editor action improves the completed draft.
Large gain denotes an improvement greater than $0.02$ under the rollout score.

\begin{table}[t]
\centering
\small
\setlength{\tabcolsep}{3.2pt}
\caption{
Stage~1 proposal diagnostics, averaged over two seeds.
Large gain denotes $\Delta > 0.02$.
All entries except Avg. gain are percentages.
}
\label{tab:stage1_bridge}
\begin{tabular}{lcccc}
\toprule
Policy & Mined-positive & Top-1 improves & Large gain & Avg. gain \\
\midrule
Synthetic recovery only & 38.3 & 50.0 & 14.8 & 0.0070 \\
Final Stage~1 editor & 84.4 & 87.5 & 57.8 & 0.0230 \\
\bottomrule
\end{tabular}
\end{table}

Mined-positive distillation substantially improves proposal quality.
The fraction of examples containing an improving proposal increases from 38.3\% to 84.4\%, while the fraction for which the editor's top action improves the draft increases from 50.0\% to 87.5\%.
Thus, synthetic recovery provides the action grammar, while mined positives shift the proposal distribution toward edits that are useful under the frozen denoiser.

\subsection{Stage 2: Selection Diagnostics}
\label{app:stage2_selection}

Stage~2 trains the selector to distinguish useful executable edits from actions that are locally valid but harmful after denoising.
We evaluate this behavior on 256 held-out editing states.

\begin{table}[t]
\centering
\small
\setlength{\tabcolsep}{5pt}
\caption{
Stage~2 selection diagnostics on 256 validation examples.
Recovered is the percentage of examples recovered after the selected edit.
}
\label{tab:stage2_diagnostics}
\begin{tabular}{lcc}
\toprule
Policy & Avg. reward & Recovered (\%) \\
\midrule
Stage~1 editor-only & -0.1270 & 40.6 \\
Stage~2 selector & -0.0250 & 63.5 \\
Oracle over candidates & 0.1002 & 79.7 \\
\bottomrule
\end{tabular}
\end{table}

The Stage~2 selector substantially improves both average reward and recovery rate over directly executing the Stage~1 editor action.
The remaining gap to the oracle indicates that useful candidates are often available but are not always identified by the learned selector.
This complements the downstream results in Section~\ref{subsec:controller_analysis}, where Stage~2 primarily prevents harmful over-editing before task-specific alignment.

\subsection{Additional Inference-Time Comparisons}
\label{app:additional_decoding}

We provide additional controlled decoding comparisons on HumanEval using the same frozen LLaDA-1.5 backbone.
These experiments test whether the gains of DELA can be explained by alternative block sizes, token ordering, or additional token-level refinement.

\begin{table}[t]
\centering
\small
\setlength{\tabcolsep}{4.0pt}
\caption{
Additional decoding comparisons on HumanEval.
All methods use the same frozen LLaDA-1.5 backbone.
}
\label{tab:additional_decoding}
\begin{tabular}{lccc}
\toprule
Method & pass@1 & NFE & Latency (s) \\
\midrule
Block size $=8$ & 37.8 & 256 & 5.09 \\
Block size $=16$ (\textsc{Direct}) & 37.2 & 256 & 5.43 \\
Block size $=32$ & 39.0 & 256 & 4.99 \\
Block size $=64$ & 37.8 & 256 & 5.04 \\
Block size $=256$ & 14.0 & 256 & 5.26 \\
Random ordering ($b=32$) & 22.6 & 256 & 5.13 \\
Tolerator-style ($T=2,k=8$) & 38.4 & 275 & 5.46 \\
Tolerator-style ($T=2,k=16$) & 38.4 & 291 & 6.19 \\
DELA-REMASK & 40.9 & 261 & 5.57 \\
DELA & 42.7 & 264 & 5.74 \\
\bottomrule
\end{tabular}
\end{table}

Changing the block size alone does not consistently improve generation quality.
A block size of 32 provides a modest improvement over the default configuration, whereas larger blocks eventually degrade substantially.
Random ordering also performs poorly, indicating that the gain is not explained simply by changing the order in which positions are generated.
Tolerator-style refinement improves over \textsc{Direct}, but requires additional denoising evaluations and remains below both DELA-REMASK and full DELA.

\subsection{Inference-Time Sensitivity}
\label{app:inference_sensitivity}

We study the sensitivity of DELA to the selector margin and the maximum number of refinement rounds on HumanEval.

\begin{table}[t]
\centering
\small
\setlength{\tabcolsep}{6pt}
\caption{
Sensitivity to the selector margin $\gamma_{\mathrm{op}}$ and refinement budget $T_{\max}$ on HumanEval.
}
\label{tab:inference_sensitivity}
\begin{tabular}{lcc}
\toprule
Setting & Value & \# Correct \\
\midrule
\multirow{4}{*}{Selector margin $\gamma_{\mathrm{op}}$}
& $0$ & 70 \\
& $0.1$ & 70 \\
& $0.5$ & 69 \\
& $\infty$ & 64 \\
\midrule
\multirow{3}{*}{Refinement budget $T_{\max}$}
& $1$ & 68 \\
& $2$ & 70 \\
& $3$ & 70 \\
\bottomrule
\end{tabular}
\end{table}

Performance is stable across $\gamma_{\mathrm{op}}\in[0,0.5]$, while disabling selector-guided overrides with $\gamma_{\mathrm{op}}=\infty$ reduces performance to the \textsc{Editor-only} result of 64 correct solutions.
Increasing the refinement budget from one to two rounds improves the number of solved problems from 68 to 70, while a third round provides no further gain.
These results indicate that DELA does not require aggressive selector overrides or long refinement trajectories in the current setting.

\subsection{Edit Behavior and Difficulty}
\label{app:edit_behavior}

We provide a more detailed characterization of the edits executed by DELA on HumanEval.
We report statistics both over the full benchmark and over the subset of examples for which the initial \textsc{Direct} draft is incorrect.

\begin{table}[t]
\centering
\small
\setlength{\tabcolsep}{4.0pt}
\caption{
Edit statistics on HumanEval.
Draft-length and changed-token statistics are computed only over examples receiving at least one edit.
}
\label{tab:edit_statistics}
\begin{tabular}{lcc}
\toprule
Metric & All examples & \textsc{Direct}-failure subset \\
\midrule
Examples receiving an edit & 15 / 164 (9.1\%) & 13 / 103 (12.6\%) \\
REMASK / INSERT / DELETE & 24 / 2 / 1 & 18 / 1 / 1 \\
Changed tokens per REMASK, median & 9 & 9 \\
Changed tokens per edited example, median & 12 & 12 \\
Changed-token fraction, median & 6.3\% & 5.9\% \\
\bottomrule
\end{tabular}
\end{table}

The executed edits are predominantly sparse and local.
Only 9.1\% of all HumanEval drafts receive an edit, and the median edited example changes 12 tokens.
REMASK dominates the executed action distribution, while INSERT and DELETE are used only occasionally for length-changing corrections.
These results support the interpretation of DELA as a sparse, local editing mechanism rather than a general-purpose rewriting system.

\subsection{Two-Phase Decoding and Cache Profiling}
\label{app:two_phase_efficiency}

We provide additional profiling results for the two-phase decoding procedure.
Because the released LLaDA implementation does not expose prefix KV caching, we implement prefix reuse during the fast-drafting phase and compare cached and uncached execution under the same frozen-prefix condition.
In both variants, finalized prefix states remain fixed as new blocks are generated. The comparison measures the cost of reusing those states, not the effect of the frozen-prefix approximation relative to full bidirectional recomputation.

\begin{table}[t]
\centering
\small
\setlength{\tabcolsep}{5pt}
\caption{
Prefix-cache profiling for 256-token fast drafting with batch size one.
}
\label{tab:cache_profile}
\begin{tabular}{lccc}
\toprule
Configuration & Latency (s) & Query positions & Attention pairs \\
\midrule
Without prefix cache & 0.414 & 5,370 & 1.51M \\
With prefix cache & 0.364 & 727 & 0.18M \\
\bottomrule
\end{tabular}
\end{table}

Prefix reuse reduces the number of processed query positions by $7.4\times$ and the number of attention pairs by $8.5\times$.
The corresponding wall-clock improvement is more modest, from 0.414\,s to 0.364\,s per sample, consistent with short-sequence batch-one inference being dominated by costs beyond attention computation.
The wall-clock benefit increases with sequence length, from $1.14\times$ at 256 tokens to $1.25\times$ at 512 tokens and $1.62\times$ at 1,024 tokens.

We also compare the default two-phase procedure with invoking global edit control throughout drafting.

\begin{table}[t]
\centering
\small
\setlength{\tabcolsep}{6pt}
\caption{
Cost of two-phase decoding compared with invoking global refinement throughout draft construction.
}
\label{tab:two_phase_profile}
\begin{tabular}{lcc}
\toprule
Procedure & Latency (s/sample) & Relative cost \\
\midrule
Two-phase drafting and refinement & 0.461 & 1.0$\times$ \\
Refinement after every block & 0.823 & 1.8$\times$ \\
Forced edit after every block & 2.466 & 5.4$\times$ \\
\bottomrule
\end{tabular}
\end{table}

The 1.8$\times$ configuration is a conservative measure of the overhead of repeatedly invoking global edit control, because the controller selects \textsc{Stop} at all block boundaries in this experiment.
The forced-edit configuration illustrates the larger cost of performing non-monotonic refinement throughout draft construction.
Together, these results show that the two-phase design avoids repeated global-controller overhead while making prefix reuse possible during the fast-drafting phase.



\end{document}